\documentclass[a4paper,11pt]{article}
\usepackage{fullpage} % can remove later if want document to be longer

\usepackage{diagbox} %new
\usepackage{slashbox} %new
\usepackage{tabularx} %new
\usepackage{bbm} %new
\usepackage{pifont}  %new
\usepackage{booktabs} %new

\usepackage{amsmath}
\usepackage{amssymb}
\usepackage{amsthm}
\usepackage{mathtools}
\usepackage{mathdots}
\usepackage{appendix}
\usepackage{graphicx}
\usepackage{enumerate}
\usepackage{microtype}
\usepackage{mleftright}
\usepackage{mdframed}
\usepackage{authblk}
\usepackage{tcolorbox}
\usepackage[english]{babel}
\usepackage[utf8]{inputenc}
\usepackage{algorithm}
\usepackage[noend]{algpseudocode}
\usepackage{xspace}
\usepackage{tikz}
\usepackage{pgfplots}
\pgfplotsset{width=8cm,compat=1.9}
\usepackage{mathdots}
\usepackage{url}
\usepackage{forest}
\forestset{default preamble={for tree={s sep+=1cm}}}
\usepackage{caption,subcaption}
\usepackage[noadjust]{cite}
\usepackage{hyperref}

\tcbset{
    rounded corners,
    colback = white,
    before skip = 0.2cm,
    after skip = 0.5cm,
    boxrule = 1.5pt,
    arc = 5pt
}

\makeatletter
\renewcommand{\ALG@name}{Protocol}
\makeatother
\allowdisplaybreaks

\graphicspath{ {./figures/} }

\usepackage{algorithm}
\usepackage[noend]{algpseudocode}

\DeclarePairedDelimiter{\ceil}{\lceil}{\rceil}

\newtheorem{theorem}{Theorem}[section]
\newtheorem*{theorem*}{Theorem}

\newtheorem{lemma}[theorem]{Lemma}

\newtheorem*{corollary*}{Corollary}

\newtheorem{observation}[theorem]{Observation}

\theoremstyle{definition}
\newtheorem{definition}[theorem]{Definition}

\newcommand{\cH}{\mathcal{H}}
\newcommand{\cX}{\mathcal{X}}

\newcommand{\cS}{\mathcal{S}}
\newcommand{\cF}{\mathcal{F}}
\newcommand{\cE}{\mathcal{E}}

\newcommand{\cV}{\mathcal{V}}
\newcommand{\cP}{\mathcal{P}}

\newcommand{\E}{\mathbb{E}}

\newcommand{\VC}{\operatorname{VC}}

\newcommand{\poly}{\operatorname{poly}}

\newcommand{\Lrn}{\mathsf{Lrn}}
\newcommand{\Adv}{\mathsf{Adv}}

\newcommand{\SOA}{\mathsf{SOA}}

\newcommand{\pac}{\mathsf{PAC}}

\newcommand{\eps}{\varepsilon}

\newcommand{\Hedge}{\mathsf{Hedge}}
\newcommand{\ClassHedge}{\mathsf{ClassHedge}}
\newcommand{\SplitVersionSpace}{\mathsf{SplitLeaf}}

\newcommand{\prop}{\mathsf{prop}}

\newcommand{\LD}{\mathtt{L}}

\newcommand{\M}{\mathtt{M}}
\newcommand{\tree}{\mathbf{T}}

\DeclareMathOperator*{\argmax}{arg\,max}
\newcommand{\idan}[1]{\textcolor{blue}{Idan: #1}}
\newif\ifanonymous
\anonymousfalse

\begin{document}
%\setcounter{Maxaffil}{3}
%\title{On the Price of Various Limitations in Multiclass Online Classification}
%\title{Bandit Feedback, Adaptive Adversaries, and Randomized Predictions}
\title{Optimal Randomized Proper Online Learning}
%\title{Tight Mistake Bounds for Randomized Proper Online Learning}

\author[1]{Zachary Chase}
\author[2]{Idan Mehalel}
%\author[2]{Idan Mehalel}
\affil[1]{Rutgers University}
\affil[2]{The Hebrew University}

\maketitle

\begin{abstract}
We prove that the optimal expected mistake bound of online learning a function class $\cH$ by a randomized proper learning algorithm is $O(\LD(\cH) \log T)$, where $\LD(\cH)$ is the Littlestone dimension of $\cH$ and $T$ is the time horizon. Our result improves upon the previously best known bound of $O(\LD(\cH) \log^6 T)$ given by \cite{daskalakis2022fast} (STOC 2022), and is optimal up to a universal constant for worst-case classes.

%We prove that the optimal expected mistake bound of online learning a function class $\cH$ by a randomized proper learning algorithm is $O(\LD(\cH) \log T)$, where $\LD(\cH)$ is the Littlestone dimension of $\cH$ and $T$ is the time horizon. This improves upon the previously best known bound of $O(\LD(\cH) \log^6 T)$ given by \cite{daskalakis2022fast} (STOC 2022). We further establish that our upper bound cannot be improved in general by more than a constant factor: for any $d \in \mathbb{N}_+$, there exists a class $\cH$ such that $\LD(\cH) = d$, and yet for any (large enough) $T$ the expected mistake bound of any randomized proper online learner for $\cH$ is $\Omega(d \log T)$. This completely resolves an open question of \cite{daskalakis2022fast}.
%Our results reveal that online and PAC learning share essentially the same price of proper learning: in the optimal error rate of both fundamental settings, an extra logarithmic factor of the sample size is always sufficient and sometimes necessary. At a technical level, the main tool utilized by our proof is a low-dimensional polynomial-size cover for a specific type of Littlestone classes that are hard to learn properly. This tool may be of independent interest.
\end{abstract}

%\pagebreak
%\setcounter{tocdepth}{2}
%\tableofcontents
%\pagebreak

\section{Introduction}
The adversarial online learning model dates back to the fundamental work of Littlestone \cite{littlestone1988learning}. It is defined as a game played between a learner $\Lrn$ and an adversary $\Adv$ for $T$ many rounds, in the presence of a concept class $\cH$ of predictors $h: \cX \to \{0,1\}$, where the domain $\cX$ is any set. At each round $t \in [T]$, the learner sends a classifier $f_t : \cX \to \{0,1\}$ to  $\Adv$, and $\Adv$ responds with a labeled example $(x_t,y_t) \in \cX \times \{0,1\}$. The learner suffers the 0/1 loss $1[f_t(x_t) \neq y_t]$. In the realizable setting considered in this paper, the adversary is restricted to output examples that are consistent with some \emph{target} $h^\star \in \cH$. That is, there must be $h^\star \in \cH$ such that $h^\star(x_t) = y_t$ for all $t \in [T]$. We denote the minimax optimal \emph{total loss} (or \emph{mistake bound}) of the learner in this setting by $\M(\cH,T)$. If there exists $M < \infty$ such that $\M(\cH,T) \leq M$ for all  $T$, we say that $\M(\cH,T)$ is \emph{learnable} with optimal mistake bound at most $M$.

The standard online learning model described above allows \emph{improper learning}: the learner can use an arbitrary classifier $f_t : \cX \to \{0,1\}$, not necessarily one belonging to $\cH$. In this work, we study the more restrictive requirement of \emph{proper learning}, under which every classifier used by the learner must belong to the given class $\cH$. Somewhat counter-intuitively, even though the input sequence is realized by some $h^\star \in \cH$, deterministic optimal proper learners can demonstrate total loss of $T$, even for extremely simple classes such as the class of singletons (see \cite{hanneke2021online} for example). Therefore, the standard proper learning model allows randomization: in round $t$, instead of choosing a classifier $h_t \in \cH$, the learner chooses a distribution $P_t$ over $\cH$, and the suffered loss is the probability that $h_t$ drawn from $P_t$ misclassifies  $x_t$. We denote the minimax optimal total loss (or \emph{expected mistake bound}) of a randomized proper learner by $\M_{\prop}(\cH,T)$.

In contrast with the improper learning model, where it is known that an optimal deterministic learner is worse than any randomized learner by at most a multiplicative factor of $2$, the proper learning scenario reveals a completely different picture: it is well known that for any learnable class $\cH$ we have $\M_{\prop}(\cH,T) = o(T)$ \cite{hanneke2021online, daskalakis2022fast}. However, prior to this work, there were no tight bounds on $\M_{\prop}(\cH,T)$ for general learnable classes. Our work shows that $\M_{\prop}(\cH,T) = O_{\cH} (\log T)$ for any learnable class $\cH$, and that there exist learnable classes for which the upper bound is tight.

The fundamental combinatorial parameter governing $\M(\cH,T)$ is the \emph{Littlestone dimension} of $\cH$, denoted by $\LD(\cH)$ (see formal definition and all relevant background in Section~\ref{sec:prem}). The classic work of Littlestone \cite{littlestone1988learning} gives an exact characterization of $\M(\cH,T)$: it proves that $\M(\cH,T) = \min\{\LD(\cH), T\}$ for all $\cH, T$. The online learner achieving this minimax optimal loss bound, termed \emph{SOA} (Standard Optimal Algorithm) is generally improper. This characterization of the improper minimax total loss via $\LD(\cH)$ provides a clean and appealing characterization, however, there are settings in which an improper prediction is not a valid output. The class $\cH$ may encode a prescribed model family, a collection of feasible decisions, or the available actions of a player in a repeated game. In these cases, a classifier outside $\cH$ may have no meaningful implementation. Proper learning also preserves the semantic, structural, or interpretability constraints that motivated the choice of $\cH$. It is therefore natural to ask how much predictive performance must be sacrificed when the learner is required to use only hypotheses from the class being learned. See, for example, \cite{braverman2026learning} for more discussion on the importance of proper learning.

\subsection{Results}

We prove the following theorem.

\begin{tcolorbox}
\begin{center}
\begin{theorem} \label{thm:main}
    For any domain $\cX$, for any class $\cH \subset \{0,1\}^{\cX}$ and for any time horizon $T \geq 2$:
    \[
    \M_{\prop}(\cH, T) = O(\LD(\cH) \log T).
    \]
    Furthermore, the upper bound is tight in the sense that for every $d \in \mathbb{N}_+$ there exists a class $\cH$ such that $\LD(\cH) = O(d)$, and yet for any $T \geq d^2$ it holds that
    \[
    \M_{\prop}(\cH, T) = \Omega(\LD(\cH) \log T).
    \]
\end{theorem}
\end{center}
\end{tcolorbox}

The best known upper bound prior to our work was $O(\LD(\cH) \log^6 T)$, given in \cite{daskalakis2022fast}. Their work considered a more general regression setting, and derived the upper bound for classification as a special case. The previously best known lower bound was the standard lower bound of $\LD(\cH)$ (for large enough $T$) for improper learning given in \cite{littlestone1988learning}.  See Table~\ref{tab:bounds} for a summary of previous and current bounds. Determining whether a polylog$(T)$ factor is necessary or not in $\M_{\prop}(\cH, T)$ was explicitly raised as an open question in \cite{daskalakis2022fast}. Theorem~\ref{thm:main} shows that a $\log T$ factor is indeed sometimes necessary, and is always sufficient. Note that while the upper bound holds for all classes, the lower bound holds only for some classes. This is not an artifact of our lower bound's proof: it is not hard to find classes for which the optimal proper mistake bound is $\LD(\cH)$, even using a deterministic proper learner.

We prove Theorem~\ref{thm:main} in Sections~\ref{sec:upper-bound}-\ref{sec:lower-bound}: Section~\ref{sec:upper-bound} contains the proof of the upper bound, and Section~\ref{sec:lower-bound} contains the proof of the lower bound.

\begin{table}[t]
  \centering
  \begin{tabular}{lll}
    \toprule
    Bound & Comments & References \\
    \midrule
    $\tilde{O} \sqrt{\LD(\cH)T}$ 
        & Agnostic regret bound 
        & \cite{rakhlin2015online,hanneke2021online} \\
    $O(\LD(\cH)\log^{6}T)$ 
        & 
        & \cite{daskalakis2022fast} \\
    $O(\LD(\cH)\log T)$ 
        & 
        & \textbf{[this work]} \\
    $\geq \LD(\cH)$ 
        & 
        & \cite{littlestone1988learning} \\
    $\Omega(\LD(\cH)\log T)$ 
        & $\forall \LD(\cH)$, $\exists \cH$
        & \textbf{[this work]} \\
    \bottomrule
  \end{tabular}
  \caption{Summary of previous and current bounds for proper online learning.
  The lower bounds hold for any large enough $T$. The comment in the last row means that for any $d \in 
  \mathbb{N}$ there exists a class $\cH$ with $\LD(\cH) = O(d)$ for which the lower bound holds. That is, the lower bound does not hold for all $\cH$. This is not an artifact of the lower bound's proof: it is not hard to find classes for which the optimal proper mistake bound is $\LD(\cH)$, even using a deterministic proper learner.}
  \label{tab:bounds}
\end{table}

\subsection{Comparison with the PAC setting}
The price of proper learning (compared to improper learning) was also studied in the classic statistical PAC learning framework, and our results have close analogues to results proved in the PAC setting. Let us briefly recall the PAC learning scenario. Given a sample size $T$, an adversary fixes a secret distribution over the domain $\cX$ and an unknown target concept $h^\star \in \cH$. The learner receives an iid labeled sample
\[
S:= (x_1, h^\star(x_1)), \ldots, (x_T,h^\star(x_T)), x_1, \ldots, x_T \sim D
\]
and in response outputs a classifier $f : \cX \to \{0,1\}$. Since in the online setting the learner is evaluated against all $T$ instances in the input sequence and we wish to compare the price of proper learning in both settings, it is convenient to measure the loss of $f$ as its expected error over $T$ many freshly drawn instances drawn iid from $D$ (instead of just one test instance, as usually done in the PAC learning setting). We denote the minimax optimal total loss in this setting as $\M^{\pac}(\cH,T)$. We also define the minimax optimal loss in the same setting when the learner is required to be proper, that is, to always output $f\in \cH$, as $\M^{\pac}_{\prop}(\cH,T)$.

The work of \cite{hanneke2016optimal} established that for any class $\cH$ and any large enough sample size $T$, $\M^{\pac}(\cH,T) = \Theta(\VC(\cH))$, where $\VC(\cH)$ denotes the VC-dimension of $\cH$, which is the PAC-analogue of the Littlestone dimension. Similarly, \cite{littlestone1988learning} showed for the online setting that $\M(\cH,T) = \LD(\cH)$ (for $T \geq \LD(\cH)$). For proper PAC learning, the classic bounds for ERM \cite{vapnik1971uniform} show that $\M^{\pac}_{\prop}(\cH,T)  = O(\VC(\cH) \log T)$. \cite{daniely2014optimal} showed that there are classes for which any optimal PAC learner must be improper, but this result is proved for the multiclass setting, where the label set size is larger than two. The work \cite{bousquet2020proper} established an analog result for standard binary classification, and also gave a quantitative lower bound of $\M^{\pac}_{\prop}(\cH,T)  = \Omega(\VC(\cH) \log T)$ for some classes and large enough $T$, matching the analysis for ERMs.

Our work completes the analog picture for online learning: While for all classes $\M(\cH,T) = \LD(\cH)$, Theorem~\ref{thm:main} establishes that $\M_{\prop}(\cH,T) = O(\LD(\cH) \log T)$ for all classes, and that this is tight for some classes.

\section{Proof sketches} \label{sec:sketch}
In this section, we sketch the main ideas of the proofs of the upper and lower bounds of Theorem~\ref{thm:main}. This section assumes some familiarity with standard concepts in learning theory and online learning theory. The reader may consider to first skim through Section~\ref{sec:prem}, where we define and state basic concepts and results, and Section~\ref{sec:upper-bound-prem}, where we present some more advanced tools used to prove the upper bound.

\subsection{Upper bound}

Our starting point is that our benchmark algorithm is the SOA: we cannot hope that a proper learner will do much better than SOA.\footnote{Even though in the proper setting we allow the learner to randomize its predictions, it is a well-known fact that randomization does not help by more than a constant factor in the improper setting.}
Therefore, it is natural to try approximating the SOA algorithm, and this is the most basic idea of the proof.
For a class $V \subset \{0,1\}^{\cX}$, we denote the classifier output by the SOA algorithm for $V$ as $\SOA_V$, which is defined as
\[
\forall x\in \cX: \SOA_V(x) = \argmax_{r \in \{0,1\}} \LD(V_{x \to r}). 
\]
In this spirit, a distribution $P$ over a set of functions $V$ is called $\epsilon$-\emph{good} if for all $x \in \cX$: 
\[
\Pr_{h \sim P}[h(x) \neq \SOA_V(x)] \leq \epsilon.
\]

For a small enough $\epsilon$, if an $\epsilon$-good distribution over the version space exists, it is intuitive that a proper learner may use it to predict and perform well. Here, we do not get into the details of how small exactly should $\epsilon$ be for this intuition to hold, and simply assume that it is small enough. The challenging part of the proof lies in the case where an $\epsilon$-good distribution does not exist.
The heart of the proof lies in the following lemma that handles this situation, and might be of interest on its own right.

\begin{tcolorbox}
\begin{center}
\begin{lemma}[Small low-dimensional cover] \label{lem:covering-informal}
    Let $V \subset \{0,1\}^{\cX}$ and $\epsilon \in (0,1)$. If there is no $\epsilon$-good distribution over $V$, then there exists a natural number $k = \poly(\LD(V)/\epsilon)$ and a family of subsets of $V$,  $\cV:= \{V_1, \ldots, V_k\}  \subset \cP(V)$ such that:
    \begin{enumerate}
        \item $\cV$ covers $V$: 
        \[
        V \subset \bigcup_{i =1}^k V_k.
        \]
        \item For all $i \in [k]$: $\LD(V_i) < \LD(V)$.
    \end{enumerate}
\end{lemma}
\end{center}
\end{tcolorbox}

The role of Lemma~\ref{lem:covering-informal} is to reduce the problem of learning $V$ to the problem of learning $k$ many function classes, where $k$ is small, and each class has a strictly smaller dimension than of $V$. After this reduction, an $\epsilon$-good distribution can be searched again for each $V_i$. If there is still some $V_i$ with no $\epsilon$-good distribution, we re-apply Lemma~\ref{lem:covering-informal} for it. Since after each application of Lemma~\ref{lem:covering-informal} the dimension decreases, eventually an $\epsilon$-good distribution must be found, since when the dimension reaches $0$ the problem becomes trivial and a $0$-good distribution must exist. This naturally induces a ``win-win" type of argument: If an $\epsilon$-good distribution exists, we can use it to predict. If not, we can apply Lemma~\ref{lem:covering-informal} and reduce the dimension at the cost of splitting the problem to a not too large number of easier sub-problems.

At a technical level, the upper bound is proved with the help of three classic, yet powerful tools:

\begin{enumerate}
    \item\label{itm:tool-hedge} \emph{The $\Hedge$ proper algorithm} of \cite{freund1997decision}, that aggregates $n$ different learning rules and performs roughly as the best learning rule up to a $\log n$ additive factor.
    \item \label{itm:tool-minimax} \emph{The minimax theorem} \cite{v1928theorie}, extended to infinite Littlestone games by \cite{hanneke2021online}.
    \item \label{itm:tool-net} \emph{The existence of small $\eps$-nets} for any distribution over a VC-class \cite{haussler1986epsilon}.
\end{enumerate}

Each of those tools is formally presented in Section~\ref{sec:upper-bound-prem}. Lemma~\ref{lem:covering-informal} is proved via a simple, yet quite sophisticated combination of Item~\ref{itm:tool-minimax} and Item~\ref{itm:tool-net}. The interested reader may find the details in the proofs of Lemma~\ref{lem:small-net} and Lemma~\ref{lem:cover-dim-drops}.

Item~\ref{itm:tool-hedge} is required as well since Lemma~\ref{lem:covering-informal} splits the version space into $k$ many version spaces, and we do not know which of them is the correct one. We therefore use Item~\ref{itm:tool-hedge} to aggregate the predictions of $k$ experts, where each expert $i$ assumes that $V_i$ is the correct version space. If Lemma~\ref{lem:covering-informal} is applied more than once, then Item~\ref{itm:tool-hedge} is used to aggregate the prediction of more than $k$ experts, but still not too many, since after every application the dimension decreases, and thus after at most $\LD(V)$ many applications the class becomes trivial. It is worth noting that the idea of aggregating predictions of several experts having different guesses on the nature of the problem was proven useful in previous works on online learning \cite{bendavid2009agnostic, long2020new, filmus2024bandit, geneson2024bounds}.

\subsection{Lower bound}
We intuitively explain how a lower bound of $\Omega(\log T)$ is obtained in the case that $T$ is fixed in advance and the Littlestone dimension is $1$. We then explain how the result can be extended to the case that $T$ is not fixed in advance, and for any Littlestone dimension.

We use the class of singletons over $n = \frac{T}{\log T}$ (assuming for simplicity that $n \in \mathbb{N}$). Formally, $\cH_n := \{h_i : i \in [n]\}$ where $h_i(j) = 1[i = j]$ for all $j \in [n]$.

As long as the learner did not see some instance labeled with $1$, every instance not yet labeled with $0$ could be labeled either with $1$ or with $0$. Therefore, the adversary only presents instances labeled with $0$. It can use this strategy for $n-1$ many instances, which is more than enough. Given this behavior of the adversary, there are essentially two reasonable strategies (or some mix of them) for the learner:
\begin{enumerate}
    \item Since every function predicts $0$ on all instances except $1$ and the adversary always labels instances with $0$, it makes sense to always predict by a uniform distribution over $\cH_n$: no matter what the adversary does, the learner is correct with probability $1-1/n$.
    \item Another reasonable strategy is to try to learn the correct target function, by using the uniform distribution only over functions $h_j$ such that $j$ was not yet labeled with $0$.
\end{enumerate}

Let us analyze the total loss suffered by each of those strategies. For the first strategy, the total loss is
\[
\frac{1}{n} T = \frac{\log T}{T} T = \log T.
\]
For the second strategy, the adversary presents the instances $1, \ldots, n-1$ one by one, labeled with $0$. Since the learner removes $h_j$ for any $j$ already labeled with $0$, its loss in round $j \in [n-1]$ is precisely $\frac{1}{n - (j-1)}$. Therefore its total loss is
\[
\sum_{j=1}^{n-1} \frac{1}{n - (j-1)} = \sum_{i=2}^n \frac{1}{i} = \Omega(\log n) = \Omega(\log (T/\log T)) = \Omega(\log T)
\]
by the standard harmonic sum analysis. So, the learner cannot escape from an $\Omega(\log T)$ total loss using those two reasonable strategies. In the formal proof, we show that this is true in general using the same ideas.

To extend the result to the case where $T$ is not fixed in advance, we use a ``gluing" technique (see \cite{chase2024deterministic} for example): we take a union of classes of the form defined above for all (large enough) values of $n$, which only increases the Littlestone dimension from $1$ to $2$. To extend the result for arbitrary values $d$ of Littlestone dimension, one can use the class of $d$-hamming spheres instead of the class of singletons. The proof with $d$-hamming spheres builds on the same ideas and adds no significant complication.

\section{Basic Preliminaries} \label{sec:prem}
\subsection{Concept classes and the Littlestone dimension}
Let $\cX$ be a nonempty set called the \emph{domain}. A \emph{concept}, or \emph{hypothesis}, is a function $h: \cX \to \{0,1\}$. A \emph{concept class} is a subset $\cH \subset \{0,1\}^\cX$.

A \emph{Littlestone tree} of depth $0$ is a single vertex. A Littlestone tree of depth $d \in \mathbb{N}_+$ is a perfect binary tree given as a function $\boldsymbol{x}: \{0,1\}^{< d} \to \cX$. A string $s \in \{0,1\}^{<d}$ indicates a path in the tree starting from the root, where $0$ indicates a left edge and $1$ indicates a right edge. $\boldsymbol{x}(s) \in \cX$ is the instance labeling the vertex the path $s$ leads to. A path $b \in \{0,1\}^{d}$ of length $d$ is called a \emph{branch}.

A concept $h$ realizes a branch $b \in \{0,1\}^{d}$ of $\boldsymbol{x}$ if $h(\boldsymbol{x}(b_{<t})) = b_t$ for all $t\in \{0, \ldots, d-1\}$. Intuitively, this means that the concept $h$ agrees with $\boldsymbol{x}$ on the sequence of labeled examples $(\boldsymbol{x}(b_{<1}), b_1), \ldots, (\boldsymbol{x}(b_{<d}), b_d)$ naturally induced by the path $b$ in the tree $\boldsymbol{x}$. The tree $\boldsymbol{x}$ is \emph{shattered} by $\cH$ if every branch in $\boldsymbol{x}$ has a concept from $\cH$ realizing it. The \emph{Littlestone dimension} of a non-empty class $\cH$, denoted by $\LD(\cH)$ is the maximal $d \in \mathsf{N}$ such that there exists a Littlestone tree of depth $d$ which is shattered by $\cH$. If there is no such maximal $d$, then $\LD(\cH) = \infty$. For an empty class $\cH = \emptyset$, we define $\LD(\cH) = -1$.

\subsection{Online learning}

\subsubsection{Improper learning}

Realizable (improper) online learning is a game that proceeds for $T \in \mathbb{N}_+$ rounds between a learner $\Lrn$ and an adversary $\Adv$, as follows. Before the game begins, a concept class $\cH$ is revealed to both players. $\Adv$ then secretly chooses an \emph{input sequence} $S$ of $T$ many \emph{examples} $S := (x_1, y_1), \ldots, (x_T,y_T)$ that is realizable by $\cH$. That is, there must exist $h \in \cH$ such that $h(x_t) = y_t$ for all $t \in [T]$. Then, the game begins, and in every round $t$:
\begin{enumerate}
    \item $\Lrn$ chooses a classifier $f_t  \in \{0,1\}^{\cX}$.
    \item The example $(x_t,y_t)$ is revealed to $\Lrn$, and $\Lrn$ suffers the \emph{loss} $1[f_t(x_t) \neq y_t]$.
\end{enumerate}
The goal of $\Lrn$ is to minimize its total loss, and $\Adv$'s goal is to maximize it. Consequently, we are interested in the minimax optimal loss bound, denoted as $\M(\cH,T)$, obtained when both players play optimally. Formally, $\Lrn$ is a function $\Lrn:(\cX \times\{0,1\})^\star \to \{0,1\}^{\cX}$, where its input in round $t$ is the sequence of $t-1$ labeled examples observed so far. Let $\cS_T(\cH)$ be the set of all input sequences of length $T$ that are realized by $\cH$. Denote the total loss bound of a learner $\Lrn$ on the input sequence $S$ by $\M(\Lrn, S)$. Then
\[
\M( \cH,T) := \inf_{\Lrn} \sup_{S\in \cS_T(\cH)} \M(\Lrn, S).
\]

The value of $\M( \cH,T)$ is well understood already in the work \cite{littlestone1988learning}.

\begin{theorem} [\cite{littlestone1988learning}] \label{thm:littlestone}
    For any class $\cH$ and time horizon $T \geq 1$, it holds that
    \[
    \M( \cH,T) = \min\{T, \LD(\cH)\}.
    \]
\end{theorem}

Furthermore, the optimal learner chooses the classifier $f_t$ of round $t$ as follows. Let $\cH^{(t)} \subset \cH$ be the subset of concepts realizing the examples $(x_1, y_1), \ldots, (x_{t-1}, y_{t-1})$. Then for every $x \in \cX$, let
\[
f_t(x) := \argmax_{r \in \{0,1\}} \LD(\cH^{(t)}_{x \to r}),
\]
where $\cH^{(t)}_{x \to r}$ consists of the functions in $\cH^{(t)}$ who classify $x$ to $r$:
\[
\cH_{x \to r} := \{h \in \cH: h(x) = r\}.
\]

\subsubsection{Proper learning}

In contrast with the setting described above, in this paper we study the problem of \emph{proper} online learning. Ideally, proper learning just means one modification to the model: instead of allowing the learner to choose any $f_t \in \{0,1\}^{\cX}$, it must choose $h_t \in \cH$. However, there are easy lower bounds showing that for some classes $\cH$ having $\LD(\cH) < \infty$ and even $\LD(\cH) = 1$, the loss bound in this ideal proper learning model can be as high as $T$ (see \cite{hanneke2021online} for example).

Therefore, it is standard in the proper learning setting to allow the algorithm to choose $h_t \in \cH$ at random,\footnote{A similar standard assumption is traditionally made also in agnostic online learning, and from the same reason. See for example \cite{shalev2014understanding}.} and the suffered loss is the mistake probability. Formally, the adversary (secretly) chooses $(x_1,y_1),\dots,(x_T,y_T)$, and then in every round $t$ of the game:
\begin{enumerate}
    \item $\Lrn$ chooses a distribution $P_t \in \Delta(\cH)$.
    \item The example $(x_t,y_t)$ is revealed to $\Lrn$, and $\Lrn$ suffers the \emph{loss}
    \[
    \Pr_{h_t \sim P_t} 1[h_t(x_t) \neq y_t].
    \]
\end{enumerate}

In the proper learning setting, the learner $\Lrn$ is a function $\Lrn:(\cX \times\{0,1\})^\star \to \Delta(\cH)$, where its input in round $t$ is the sequence of $t-1$ labeled examples observed so far. The notation $\Delta(\cdot)$ refers to the set of all finite support distributions over some set. We consider only finite support distributions since this suffices for our algorithm, and avoids measure theory formalism.

For a class $\cH$ and a time horizon $T$, we denote the minimax optimal total loss by $\M_{\prop}(\cH,T)$, which is defined as $\M(\cH,T)$, with the exception that the infimum is taken over proper randomized learners. We are interested in bounding $\M_{\prop}(\cH,T)$ in terms of $\LD(\cH)$ and $T$.

\section{The upper bound}\label{sec:upper-bound}

\begin{theorem} \label{thm:main-upper-bound}
    For any class $\cH$ with $\LD(\cH) \geq 1$ and any $T \geq 2$ , we have
    \[
    \M_{\prop}(\cH,T) = O(\LD(\cH) \log T).
    \]
\end{theorem}

\subsection{Preliminaries for the proof} \label{sec:upper-bound-prem}

We will use the following version of the minimax theorem for infinite Littlestone games of \cite{hanneke2021online}, which is a special case of their Corollary 7.

\begin{theorem}[\cite{hanneke2021online}] \label{thm:minimax}
    Let $\cH \subset \{0,1\}^\cX$ be a class with $\LD(\cH) < \infty$. Then
    \[
    \inf_{P \in \Delta(\cH)} \sup_{x \in \cX} \Pr_{h \sim P}[h(x) = 1] = \sup_{Q \in \Delta(\cX)} \inf_{h \in \cH}  \Pr_{x \sim Q}[h(x) = 1].
    \]
\end{theorem}

The second tool we will use is the classic \emph{Hedge} framework and algorithm of \cite{freund1997decision}. The hedge framework is defined as follows. The following game is played between $\Lrn$ and $\Adv$ for $T$ many rounds in the presence of $N$ many experts, where each expert $i$ is referred to as $i$. In each round $t$, $\Lrn$ chooses a distribution $P_t \in \Delta([N])$ over the experts. Then, the losses $\ell_{t,i} \in [0,1]$ of all experts $i \in [N]$ are revealed, and $\Lrn$ suffers the loss
\[
\E_{i \sim P_t}[ \ell_{t,i}] = \sum_{i \in [N]} P_t(i) \ell_{t,i}.
\]
The learner's goal is to minimize the total loss suffered in the game. There is no limitation on how the losses are generated. The work of \cite{freund1997decision} proved the following guarantee.

\begin{theorem}[\cite{freund1997decision}] \label{thm:hedge}
    Consider the Hedge framework, with the additional guarantee that there exists $K \in \mathbb{R}$ and an expert $i^\star \in [N]$ such that
    \[
    \sum_{t \in [T]} \ell_{t,i^\star} \leq K.
    \]
    Then, there exists an algorithm (coined as $\Hedge$) with total loss
    \[
    \sum_{t \in [T]} \sum_{i \in [N]} P_t(i) \ell_{t,i} = O(\log N + K).
    \]
\end{theorem}

We will also rely on the classic $\epsilon$-net theorem. Below, to avoid measure theory formalism, we consider only finite-support distributions.

\begin{definition}[$\epsilon$-net]
    Let $\cH \subset \{0,1\}^\cX$ be a class, and let $Q$ be some finite-support probability distribution defined over $\cX$. A subset $N \subset \cX$ is an $\epsilon$-net for $\cH, Q$ if for any $h \in \cH$ it holds that
    \[
    \Pr_{x \sim Q} [h(x) = 1] \geq \epsilon \implies \exists x \in N: h(x) = 1.
    \]
\end{definition}

\begin{theorem}[\cite{haussler1986epsilon}] \label{thm:net}
    Let $\cH \subset \{0,1\}^\cX$ be a class with $0 < \LD(\cH) < \infty$. For any finite-support probability distribution $Q$ defined over $\cX$, and for any $\epsilon > 0$ there exists an $\epsilon$-net for $\cH, Q$ of size $O \mleft (\frac{\LD(\cH)}{\epsilon} \log \mleft (\frac{\LD(\cH)}{\epsilon} \mright)\mright)$.
\end{theorem}

Theorem~\ref{thm:net} was originally stated with the VC-dimension instead of the Littlestone dimension, as stated above. However, since the Littlestone dimension upper bounds the VC-dimension, the above formulation of the theorem is correct as well.

\subsection{Proof of the upper bound}
Fix the class $\cH$ and the horizon $T \in \mathbb{N}$, and denote $\LD(\cH) := d < \infty$. Throughout the proof, we may assume that $T = \Omega(d)$. In the complementing case, the upper bound is trivial.  The assumption that $T$ is given and fixed in advance can be removed via a standard exponential doubling trick (see~\cite[Section 4.1]{auer2010ucb} for example).

The proof constructs a small set of experts which are proper learning algorithms for $\cH$ with horizon $T$, such that the assigned loss for each expert in round $t$ is an upper bound on its algorithm's loss on the example $(x_t,y_t)$, and such that the total loss of at least one expert is small. Then, the $\Hedge$ algorithm is being executed over this set of experts. This is done by the $\ClassHedge$ algorithm described in Figure~\ref{fig:main-alg}.

We now present some notation and the algorithm $\ClassHedge$. Let $V \subset \cH$.
Define the function $y_V : \cX \to \{0,1\}$ as
\[
y_V(x) := \argmax_{r \in \{0,1\}} \LD(V_{x \to r})
\]
for all $x\in \cX$.

Algorithm $\ClassHedge$ is instanced by a set $\cE$ of experts. Each expert is a function of a version space $V \subset \cH$ associated with it. The version space associated with $e \in \cE$ is denoted by $V(e)$. For algorithmic reasons, each expert is also associated with a leaf in a tree $\tree$ maintained throughout the execution of $\ClassHedge$. The leaf associated with $e \in \cE$ is denoted by $\ell(e)$. Each leaf $\ell$ in the tree is associated with a version space denoted by $V(\ell)$. For every expert $e$, we always have $V(e) = V(\ell(e))$. That is, an expert is associated with the version space associated with its leaf.  We stress that there could be multiple experts associated with the same leaf. In fact, at first $\tree$ contains only a single node, which is the root $r$. We initialize $V(r) = \cH$, and $\ell(e) = r$ for all $e \in \cE$. Let $\cE(\ell)$ be the set of experts associated with the leaf $\ell$.

In each round $t$, each expert presents a distribution $P_{t,e}$ over $\cH$ (which is supported over $V(e)$). Algorithm $\ClassHedge$ predicts with a distribution $P_t$ over $\cH$, which is a mixture of $\{P_{t,e}\}_{e \in E}$, where the weight of every expert in the mixture is given by the weights produced by $\Hedge$.

\begin{figure}
    \centering
    \begin{tcolorbox}
    \begin{center}
        \textsc{$\ClassHedge$}
    \end{center}
    \textbf{Initialize:} Let the set of experts $\cE$ be of size $T^{3d}$, and set the initial weights of the experts as in $\Hedge$. Initialize a tree $\tree$ with only a root $r$. Set $V(r) := \cH$, and set $\ell(e) = r$ for all $e \in \cE$,
    \\
    \textbf{for $t=1,\ldots, T$:}
    \begin{enumerate}
       \item While there is a leaf $\ell \in \tree$ such that $P_{t,e}$ is not yet defined for $e \in \cE(\ell)$:
       \begin{enumerate}
           \item If $V(\ell) = \emptyset$, set $P_{t,e}$ to be an arbitrary distribution over $\cH$ for all $e \in \cE(\ell)$.
           \item Else, if there is a finite support distribution $P$ over $V(\ell)$ such that
           \[
           \forall x \in \cX: \Pr_{h \sim P}[h(x) \neq y_{V(\ell)}(x)] \leq 2/T,
           \]
           set $P_{t,e} := P$ for all $e \in \cE(\ell)$.
           \item Else, if there is no such distribution, execute $\SplitVersionSpace(\ell)$.        
       \end{enumerate}
       \item Define a distribution $P_t$ over $\cH$ which is given by the distribution over the experts according to their weights as in $\Hedge$. The prediction of an expert $e$ is given by $P_{t,e}$.
       \item Receive current example $(x_t,y_t)$.
       \item For every leaf $\ell \in \tree$:
       \begin{enumerate}
           \item If $V(\ell) = \emptyset$, set $\ell_{t,e} := 1$ for all $e \in \cE(\ell)$.
           \item Else, if $y_t = y_{V(\ell)}(x_t)$: set $\ell_{t,e} := 2/T$ for all $e \in \cE(\ell)$.
           \item Else, if $y_t \neq y_{V(\ell)}(x_t)$: set $\ell_{t,e} := 1$ for all $e \in \cE(\ell)$.
           \item Update $V(\ell)$ to $V(\ell ):= \{h \in V(\ell): h(x_t) = y_t\}$.
       \end{enumerate}
       \item Update experts' weights as in $\Hedge$.
    \end{enumerate}
    \end{tcolorbox}
    \caption{Main algorithm} 
    \label{fig:main-alg}
\end{figure}

\begin{figure}
    \centering
    \begin{tcolorbox}
    \begin{center}
        \textsc{$\SplitVersionSpace$}
    \end{center}
    \textbf{Input:} A leaf $ \ell \in \tree$.
    \\
    \begin{enumerate}
        \item \label{itm:find-net} Find a subset $N \subset \cX$ of size $T^3$ such that for any $h \in V(\ell)$ it holds that there exists $x \in N$ for which $h(x) \neq y_{V(\ell)}(x)$. If such $N$ does not exist, return an error.
        \item \label{itm:split} Split $\cE(\ell)$ to $T^3$ disjoint subsets of equal sizes, each associated with a distinct $x \in N$. If $|\cE(\ell)| \neq T^{3k}$ for some $k \in \mathbb{N_+}$, return an error. Denote the subset of $\cE(\ell)$ associated with $x$ by $\cE_x(\ell)$.
        \item Turn $\ell$ to an internal vertex by attaching to $\ell$ $T^3$ many leaves, each associated with a distinct $x \in N$. Denote the new leaf associated with $x$ by $\ell_x$.
        \item For any $x \in N$ update $\cE(\ell_x) := \cE_x(\ell)$, and $V(\ell_x) := V_x(\ell)$, where:
        \[
        V_x(\ell) := \{h \in V(\ell): h(x) \neq y_{V(\ell)}(x)\}.
        \]
    \end{enumerate}
    \end{tcolorbox}
    \caption{Leaf splitting algorithm} 
    \label{fig:covering-alg}
\end{figure}

We now prove that $\ClassHedge$ indeed has a total loss bound of $O(d \log T)$, by a series of arguments.

\begin{lemma}\label{lem:covering}
    At all times, for every $h\in \cH$ which is consistent with the input, there exists $\ell \in \tree$ such that $h \in V(\ell)$.
\end{lemma}

\begin{proof}
    Since we initialize $V(r) = \cH$, it is not difficult to see that it suffices to show that for any $\ell \in \tree$, if $N$ is the subset used in $\SplitVersionSpace(\ell)$ to split $\ell$, then
    \[
    V(\ell) \subset \bigcup_{x \in N} V_x(\ell).
    \]
    Let $h\in V(\ell)$. By definition of $N$, there exists $x \in N$ such that $h(x) \neq y_{V(\ell)}(x)$, that is, such that $h \in V_x(\ell)$.
\end{proof}

\begin{lemma}\label{lem:small-net}
    For any $\ell$ for which $\SplitVersionSpace$ is executed for, the set $N$ defined in $\SplitVersionSpace(\ell)$ exists. That is, Item~\ref{itm:find-net} of $\SplitVersionSpace$ never returns an error.
\end{lemma}

\begin{proof}
    We assume for simplicity that $|\cX| \geq T^3$ (otherwise the subset $N$ can be a multiset, which does not hurt the algorithm).
    Since $\SplitVersionSpace(\ell)$ is executed, it is implied that $V(\ell)$ is non-empty and that there is no finite support distribution $P$ over $V(\ell)$ such that
    \[
        \forall x \in \cX: \Pr_{h \sim P}[h(x) \neq y_{V(\ell)}(x)] \leq 2/T,
    \]
    That is, for any finite support distribution $P$ over $V(\ell)$ there exists $x \in \cX$ for which
    \[
        \Pr_{h \sim P}[h(x) \neq y_{V(\ell)}(x)] > 2/T.
    \]
    For any $h \in \cH$, define the function $b_h$ for all $x \in \cX$ as
    \[
    b_h(x) := 1[h(x) \neq y_{V(\ell)}(x)],
    \]
    and so we can write that for any finite support distribution $P$ over $V(\ell)$ there exists $x \in \cX$ for which
    \begin{equation} \label{eq:minimax-condition}
        \Pr_{h \sim P}[b_h(x) = 1] > 2/T.
    \end{equation}
    Furthermore, the class $B_{V(\ell)}: = \{b_h: h \in V(\ell)\}$ satisfies $\LD(B_{V(\ell)}) = \LD(V(\ell))$. Indeed, note that for any $x \in \cX, h\in V(\ell)$ we have $b_h(x) = h(x) \oplus y_{V(\ell)}(x)$. That is, for $x$ with $y_{V(\ell)}(x) = 1$ we have $b_h(x) = 1 - h(x)$, and for $x$ with $y_{V(\ell)}(x) = 0$ we have $b_h(x) = h(x)$. Thus, any Littlestone tree for $V(\ell)$ can serve as a Littlestone tree for  $B_{V(\ell)}$ by switching the labels on the left and right outgoing edges of every node labeled with $x \in \cX$ for which $y_{V(\ell)}(x) = 1$. Therefore, the minimax Theorem~\ref{thm:minimax} holds for $B_{V(\ell)}$, and thus \eqref{eq:minimax-condition} implies that there exists a finite support distribution $Q$ over $\cX$ such that for any $b_h \in B_{V(\ell)}$ we have
    \begin{equation}\label{eq:net-condition}
        \Pr_{x \sim Q} [b_h(x) = 1] > 1/T.
    \end{equation}
    Now, Theorem~\ref{thm:net} implies that there exists a $(1/T)$-net $N$ for $B_{V(\ell)}, Q$ of size $T^3$. Together with \eqref{eq:net-condition}, this implies that  for any $b_h \in B_{V(\ell)}$ there exists $x \in N$ such that   $b_h(x) = 1$, that is, for any $h \in V(\ell)$ there exists $x \in N$ such that $h(x) \neq y_{V(\ell)}(x)$, as required.
\end{proof}

\begin{lemma} \label{lem:cover-dim-drops}
    Let $\ell \in \tree$, and let $N$ as defined in $\SplitVersionSpace(\ell)$. Then for every $x \in N$ we have
    \[
    \LD(V_x(\ell)) < \LD(V(\ell)).
    \]
\end{lemma}

\begin{proof}
    Suppose that for some $x \in N$ we have $\LD(V_x(\ell)) \geq \LD(V(\ell))$. Then we can construct a Littlestone tree of depth $\LD(V(\ell)) + 1$ which is shattered by $V(\ell)$ (which is of course a contradiction), as follows: Let the root be labeled with $x$. By assumption and by definition of $V_x(\ell)$, we have
    \[
    \LD(V(\ell)_{x \to y_{V(\ell)}(x)})
    \geq
    \LD(V(\ell)_{x \to 1-y_{V(\ell)}(x)})
    =
    \LD(V_x(\ell))
    \geq
    \LD(V(\ell)),
    \]
    which concludes the proof. Therefore, $\LD(V_x(\ell)) < \LD(V(\ell))$ for all $x \in N$.
\end{proof}

\begin{lemma}\label{lem:experts-number}
    Whenever $\SplitVersionSpace(\ell)$ is executed, there exists $k \in \mathbb{N}_+$, such that $|\cE(\ell)| = T^{3k}$. That is, Item~\ref{itm:split} of $\SplitVersionSpace$ does not return an error.
\end{lemma}

\begin{proof}
    We first have only the root $r$, and thus $|\cE(r)| = T^{3d}$. Each execution of $\SplitVersionSpace(\ell)$ splits $\cE(\ell)$ to $T^3$ disjoint subsets. Thus, it suffices to show that if $|\cE(\ell)| = 1$, then $\SplitVersionSpace$ will not be called with $\ell$.
    
    In other words, it suffices to show that $\tree$ never reaches depth $d+1$, in any of its branches. Indeed, Lemma~\ref{lem:cover-dim-drops} establishes that for any $\ell_x$ attached to $\ell$ when $\SplitVersionSpace(\ell)$ is called, we have $\LD(V(\ell_x)) < \LD(V(\ell))$. Therefore, if for some $\ell \in \tree$ we have $|\cE(\ell)| = 1$, it is implied that $\ell$ is at depth $d$, which due to Lemma~\ref{lem:cover-dim-drops} implies that $\LD(V(\ell)) = 0$. We can now observe that the condition that fires the execution of $\SplitVersionSpace(\ell)$ cannot be satisfied for $\ell$ with $\LD(V(\ell)) = 0$, since $\LD(V(\ell)) = 0$ implies $|V(\ell)| = 1$. 
\end{proof}

\begin{lemma} \label{lem:good-experts}
    There exists $e \in \cE$ such that
    \[
    \sum_{t \in [T]} \ell_{t, e} \leq \LD(\cH) + 2. 
    \]
\end{lemma}

\begin{proof}
    Let $h \in \cH$ such that $h$ is consistent with the input sequence. Let $\ell \in \tree$ such that $h \in V(\ell)$ in round $T$, which exists by Lemma~\ref{lem:covering}. By Lemma~\ref{lem:experts-number}, we also have $V(\ell) \neq \emptyset$, so let $e \in \cE(\ell)$. In every round $t$, there are two possibilities for $\ell_{t,e}$. If $y_t = y_{V(e)}(x_t)$, then $\ell_{t,e} = 2/T$. Thus the total loss from such rounds is at most $\frac{2}{T} T = 2$. In the other case, we have $y_t \neq y_{V(e)}(x_t)$ and then $\ell_{t,e} = 1$. By definition of $y_{V(e)}$, the Littlestone dimension of the version space of $e$ will decrease by at least $1$ after $\ell(e)$ updates its version space to be consistent with the example $(x_t,y_t)$. Therefore, this case can happen at most $d$ many times, and the proof is completed.  
\end{proof}

We may now prove the following lemma that implies Theorem~\ref{thm:main-upper-bound}.

\begin{lemma} \label{lem:upper-bound-main}
    Algorithm $\ClassHedge$ is a proper online learner such that for any $\cH$ with $\LD(\cH) \geq 1$ and for any $T \geq 2$, we have $\M(\ClassHedge, S) = O(\LD(\cH) \log T)$ for all $S \in \cS_T(\cH)$.
\end{lemma}

\begin{proof}
    First, $\ClassHedge$ is a proper learner since in every round it defines a distribution $P_t$ over proper learners, and thus proper itself. We now prove its loss guarantee.

    Now, note that for any round $t$ and for any $e \in \cE$ we have
    \begin{equation}\label{eq:loss-lower-bound}
       \sum_{h \in \cH} P_{t,e}(h) 1[h(x_t) \neq y_t] \leq \ell_{t,e}.
    \end{equation}
    Indeed, for the cases that $V(e) = \emptyset$ or $y_t \neq y_{V(e)}(x_t)$ in round $t$ this is trivial. For the case $y_t = y_{V(e)}(x_t)$, the definition of $P_{t,e}$ implies that $\Pr_{h \sim P_{t,e}} [h(x_t) \neq y_t] \leq 2/T = \ell_{t,e}$.
    Let $E_{t}$ be the distribution over experts that $\ClassHedge$ uses in round $t$.
    Now, by definition of $\ClassHedge$, its loss on $S$ is
    \begin{align*}
        M(\ClassHedge, S) &= \sum_t \sum_{h \in \cH} P_t(h)1 [h(x_t) \neq y_t] \\
                          &= \sum_t \sum_{e \in \cE} E_t(e)  \sum_{h \in \cH} P_{t,e}(h) 1[h(x_t) \neq y_t] \\
                          &\leq \sum_t \sum_{e \in \cE} E_t(e) \ell_{t,e},
    \end{align*}
    where the final inequality is by \eqref{eq:loss-lower-bound}. Now, $\ClassHedge$ is the $\Hedge$ algorithm executed on $T^{3d}$ experts, with the guarantee that there exists an expert with total loss at most $d+2$, due to Lemma~\ref{lem:good-experts}. Theorem~\ref{thm:hedge} implies that the quantity $\sum_t \sum_{e \in \cE} E_t(e) \ell_{t,e}$ is bounded as
    \[
    \sum_t \sum_{e \in \cE} E_t(e) \ell_{t,e} = O( \log T^{3d} + d) = O(d \log T).
    \]
    This concludes the proof.
\end{proof}

\section{The Lower bound} \label{sec:lower-bound}
We prove the following lower bound.

\begin{theorem} \label{thm:lower-bound}
    For any $d \in \mathbb{N_+}$, there exists a class $\cH_d$ such that $\LD(\cH_d) = O(d)$, and for any $T \geq d^2$:
    \[
    \M_{\prop}(\cH_d,T) = \Omega(d \log T).
    \]
\end{theorem}

Let us first define for every $d$ the class $\cH_d$ used in Theorem~\ref{thm:lower-bound}. The class is defined as a union of $d$-hamming spheres over a block $[n]$, for all block size $n \geq 2d$. The block of size $n$ is referred to as block $n$. We need varying block sizes to handle all possible horizons $T$ in a single class. This is similar to the gluing technique used e.g.\ in \cite{chase2024deterministic}.

Formally, for every block size $n \geq 2d$ define
\[
\cH_{d,n} := \{h_{n,S} : S \subset [n], |S| = d \},
\]
and $\cH_d$ is the union
\[
\cH_d := \bigcup_{n \geq 2d} \cH_{d,n}.
\]
The domain $\cH_d$ is defined over is
\[
\cX_d := \{I_n : n \geq 2d\} \cup \{x_{n,i} : n \geq 2d, i \in [n]\}.
\]
The role of $I_n$ is to indicate the block, and for every $n \geq 2d$ the instances $x_{n,1}, \ldots, x_{n,n}$ form the block of size $n$. We denote the set of instances used in block $n$ as
\[
\cX_{d,n} = \{I_n\} \cup \{x_{n,i} : i \in [n]\},
\]
such that $\cX_d = \bigcup_{n \geq 2d} \cX_{d,n}$.

Let us formally define $\cH_{d,n}$. For all $n \geq 2d$ and $S \subset [n]$ define $h_{n,S}(I_m) := 1[m = n]$ for all $m \geq 2d$, and $h_{n,S}(x_{m,i}) = 1[m=n \text{ and } i\in S]$. It is not difficult to verify the following

\begin{observation} \label{obs:hd-ldim}
    We have $\LD(\cH_d) = d + 1$.
\end{observation}

We may now prove the lower bound. In the proof, it is convenient to use an adaptive adversary $\Adv$ that is allowed to choose the input sequence on the fly. However, adaptive and oblivious adversaries are in fact equivalent since the distribution $P_t$ used by the learner in every round $t$ is a deterministic function of the input sequence up to round $t-1$.

In fact, using a somewhat more involved argument, one can prove that the same lower bound holds even if the adversary is \emph{stochastic}, which is even weaker than an adversarial oblivious adversary.
%Proving tight upper bounds against this weaker adversary remains open. More details can be found in Section~\ref{sec:future}.

\begin{proof}[Proof of Theorem~\ref{thm:lower-bound}]
    Fix $d \in \mathbb{N}_+$ and $T \geq d^2$. The proof requires that $T$ is larger than some universal constant $C >0$. We therefore assume that $d \geq \sqrt{C}$. This does not invalidate the proof for $d < \sqrt{C}$, as for such values of $d$ we use the class $\cH_d$ with $d = \ceil*{\sqrt{C}}$. We first fix the block size as
    \[
    n := \frac{T}{\log T},
    \]
    assuming for simplicity that $n \in \mathbb{N}$. We have $n \geq 2d$ for sufficiently large $C$, thus a block of size $n$ exists.

    $\Adv$ uses only instances from $\cX_{d,n}$ in the input sequence. It maintains a set $U_t \subset[n]$ of instances $\{x_{n,i} : i \in U_t\}$ that were not yet labeled with $0$ in the beginning of round $t$. Crucially, as long as $|U_t| > d$, $\Adv$ is allowed to use any coordinate instance $x_{n,i}$ labeled with $0$ as an example, while keeping the input sequence realizable. Initially $U_1 = [n]$ and we assume that $|U_t|  \geq d$ at all times. To make this assumption hold, once $|U_t|$ is exactly $d$, $\Adv$ produces arbitrary examples that are consistent with $h_{n,U_t}$, and we assume that the learner suffers no further loss. We also denote the set of indices $i$ such that the example $(x_{n,i},0)$ already appeared in the input sequence up to round $t-1$ as $U_t^c := [n] \backslash U_t$.
    %We denote the version space matching to $U_t$ as
    %\[
    %V(U_t) := \{h_{n,S} \in \cH_{d,n}: S \subset U_t\}.
    %\]
    %Likewise, let $V^c(U_t) := \cH_{d,n} \backslash V(U_t)$.
    
    Fix the round $t$, and let $P$ be the distribution used by $\Lrn$ in round $t$. Let $p_{b} := P(\cH_d \backslash \cH_{d,n})$ be the total mass of $P$ outside of the block's functions $\cH_{d,n}$. We have two cases. In the first case, which we call case 1, we have $p_{b} \geq 1/2$. In this case, $\Adv$ sets the example of round $t$ to $(I_n,1)$, and the learner suffers loss at least
    \begin{equation} \label{eq:loss-bad}
        p_b \geq 1/2,
    \end{equation}
    without elimination of any element of $U_t$. Thus, in this case we set $U_{t+1} := U_t$. In the complementing case, which we call case 2, the total mass of the functions in $\cH_{d,n}$, denoted as $p_g:= P(\cH_{d,n})$, is more than $1/2$. For any $i \in [n]$, let $p_i$ be the total mass of functions giving $1$ to $x_{n,i}$:
    \[
    p_i := P(\{h_{n,S} : S \subset [n], i \in S\}).
    \]
    By definition of $p_i$ and the assumption $p_g > 1/2$ we have
    \[
    \sum_{i=1}^n p_i = d p_g > d/2.
    \]
    Indeed, every $h \in \cH_{d,n}$ appears in the set $\{h_{n,S} : S \subset [n], i \in S\}$ for precisely $d$ many coordinates $i$, thus $\sum_{i=1}^n p_i$ in fact sums the total mass of $\cH_{d,n}$ for precisely $d$ many times.

    We now handle two different cases that case 2 splits to. In the first case, which we call case 2a, we have $\sum_{i \in U_t^c} p_i \geq d/4$. When this is the case, $\Adv$ identifies $i \in U_t^c$ who maximizes $p_i$ and sets the example in round $t$ to $(x_{n,i}, 0)$. This example does not eliminate an element from $U_t$ and thus we set $U_{t+1} := U_t$. By definition of $p_i$ the learner suffers loss at least
    \begin{equation}\label{eq:loss-out-of-vs}
        p_i \geq \frac{d}{4 |U_t^c|} \geq \frac{d}{4 n}.
    \end{equation}
    In the complementing case, which we call case 2b, we have $\sum_{i \in U_t} p_i \geq d/4$. $\Adv$ identifies $i \in U_t$ who maximizes $p_i$ and sets the example in round $t$ to $(x_{n,i}, 0)$. By definition of $p_i$ the learner suffers loss at least
    \begin{equation}\label{eq:loss-in-vs}
        p_i \geq \frac{d}{4 |U_t|}.
    \end{equation}
    In this case, the index $i$ should be eliminated from $U_t$, so we set $U_{t+1} := U_t \backslash \{i\}$.

    We now lower bound the total loss suffered by the learner.
    %Let $T' \leq T$ be the maximal round in which $|U_{T'}| > d$. For every $t$, we have $|U_{t+1}| \geq |U_t| - 1$, which implies that $T' \geq n-1$. From now on, we discuss only the rounds $1, \ldots T'$. Let us first suppose that at least half of those rounds are case $1$. Then the total loss suffered by the learner is at least $\frac{1}{2}\frac{1}{2} (n-1) \gg \frac{1}{4}d \log T$. Therefore, suppose that more than half of the rounds are case 2.
    We consider two cases. In the first case, there exists $T'$ such that $|U_{T'}| = d$. This means that case 2b happened for all possible sizes of $U_t$: $m \in \{n, \ldots, d+1\}$. The total loss originating only from case 2b rounds is thus at least
    \[
    \sum_{m=d+1}^n \frac{d}{4 m} \geq \frac{1}{8} d \log \frac{n}{d} = \frac{1}{8} d \log \frac{T}{d \log T} = \Omega(d \log T),
    \]
    due to \eqref{eq:loss-in-vs}, where the final equality holds since $d\geq \sqrt{C}$ where $C$ is sufficiently large.

    In the complementing case, $|U_T| > d$, which means that at most $n$ many rounds were case 2b, and that at least $T -n \geq T/2$ rounds were case 1 or case 2a. In both cases, the loss per round is at least $\frac{d}{4n}$ due to \eqref{eq:loss-bad}, \eqref{eq:loss-out-of-vs}. Overall, this sums up to total loss of at least
    \[
    \frac{T}{2} \frac{d}{4n} = \frac{T}{2} \frac{d\log T}{4T} = \Omega(d \log T),
    \]
    as stated.
\end{proof}

\section*{Acknowledgments}
Idan Mehalel is supported by the European Research Council (ERC) under the European Union’s Horizon 2022 research and innovation program (grant agreement No. 101041711), the Israel Science Foundation (grant number 2258/19), and the Simons Foundation (as part of the Collaboration on the Mathematical and Scientific Foundations of Deep Learning).

%\section*{Acknowledgments}

\bibliographystyle{alphaurl}
\bibliography{bib.bib}

\end{document}